\documentclass[english, 10pt]{article}
\usepackage[margin=1in]{geometry}
\usepackage{multirow}
\usepackage{amsthm}
\usepackage{amsmath}
\usepackage{amsfonts}
\usepackage{amssymb}
\usepackage{bm}
\usepackage{natbib}
\usepackage{authblk}
\usepackage{color}
\usepackage[dvipsnames]{xcolor}
\usepackage{caption}
\usepackage{subcaption}
\usepackage{wrapfig}
\usepackage{graphicx}
\usepackage{array}
\usepackage{setspace}
\usepackage{algorithm}
\usepackage{algpseudocode}
\usepackage[bookmarks=false, colorlinks=true, citecolor=MidnightBlue, linkcolor=MidnightBlue, urlcolor=MidnightBlue]{hyperref}
\usepackage{setspace}
\usepackage{bbm}
\usepackage{thm-restate}
\usepackage{comment}
\algnewcommand{\LineComment}[1]{\State \(\triangleright\) #1}

\usepackage{xr}

\makeatletter

\usepackage{verbatim}   % mult-line (block) comments

\usepackage{microtype}
\usepackage{mathpazo} % math & rm
\usepackage{courier} % tt
\normalfont
\usepackage[T1]{fontenc}
\renewcommand{\leq}{\leqslant}
\renewcommand{\le}{\leqslant}
\renewcommand{\geq}{\geqslant}
\renewcommand{\ge}{\geqslant}

  \theoremstyle{plain}
  \newtheorem{Theorem}{\protect\theoremname}
  \theoremstyle{plain}
  \newtheorem*{Theorem*}{\protect\theoremname}
  \theoremstyle{plain}
  \newtheorem{proposition}[Theorem]{\protect\propositionname}
  \theoremstyle{plain}
  \newtheorem*{prop*}{\protect\propositionname}
  \theoremstyle{plain}
  \newtheorem{lemma}[Theorem]{\protect\lemmaname}
  \theoremstyle{plain}
  \newtheorem*{lemma*}{\protect\lemmaname}  
  \theoremstyle{plain}
  \newtheorem{definition}{\protect\definitionname}
  \theoremstyle{plain}
  \newtheorem{corollary}[Theorem]{\protect\corollaryname}
  \theoremstyle{plain}
  \newtheorem{example}{\protect\examplename}
 \theoremstyle{plain}

\theoremstyle{plain}

\theoremstyle{plain}
\newtheorem{assumption}{\protect\assumptionname}
\theoremstyle{plain}

 \theoremstyle{plain}
\newtheorem*{model*}{Model}

\makeatother

\usepackage{complexity}     % to use poly

\newcommand{\ind}{\mathbb{1}}
\renewcommand{\E}{\mathbb{E}}

\newcommand{\Prob}{\mathbb{P}}

\usepackage{babel}
\providecommand{\assumptionname}{Assumption}
\providecommand{\definitionname}{Definition}
\providecommand{\lemmaname}{Lemma}
\providecommand{\propositionname}{Proposition}
\providecommand{\corollaryname}{Corollary}
\providecommand{\examplename}{Example}
\providecommand{\factname}{Fact}
\providecommand{\conditionname}{Condition}
\providecommand{\theoremname}{Theorem}

\DeclareMathOperator*{\argmax}{arg\,max}  % in your preamble
\DeclareMathOperator*{\argmin}{arg\,min}  % in your preamble 

\usepackage{enumitem}   % added in order to use \begin{enumerate}[label=\Alph*)]

\usepackage{bm}
\newcommand{\thetabf}{\bm{\theta}}

\usepackage{cleveref}

\newcommand{\lambdabf}{\boldsymbol{\lambda}}

\newcommand{\cl}{\textnormal{cl}}

\newcommand\blfootnote[1]{%
  \begingroup
  \renewcommand\thefootnote{}\footnote{#1}%
  \addtocounter{footnote}{-1}%
  \endgroup
}

\begin{document}

\title{Admissibility of Completely Randomized Trials:\\
A Large-Deviation Approach} 
\author{Guido Imbens}
\author{Chao Qin}
\author{Stefan Wager}
\affil{Stanford Graduate School of Business}
\pdfoutput=1
\maketitle
\onehalfspacing
%\doublespacing
\begin{abstract}
When an experimenter has the option of running an adaptive trial, is it admissible to ignore this option and run a non-adaptive trial instead? We provide a negative answer to this question in the best-arm identification problem, where the experimenter aims to allocate measurement efforts judiciously to confidently deploy the most effective treatment arm.
We find that, whenever there are at least three treatment arms, there exist simple adaptive designs that universally and strictly dominate non-adaptive completely randomized trials.
This dominance is characterized by a notion called efficiency exponent, which quantifies a design's statistical efficiency when the experimental sample is large.
Our analysis focuses on the class of batched arm elimination designs, which progressively eliminate underperforming arms at pre-specified batch intervals.
We characterize simple sufficient conditions under which these designs universally and strictly dominate completely randomized trials.
These results resolve the second open problem posed in \citet{qin2022open}.
\end{abstract}

\section{Introduction}
Randomized control trials (RCTs)\blfootnote{\hspace{-5mm}A preliminary version of this work will appear as a one-page abstract at the 26th ACM Conference on Economics and Computation (EC'25). 
We are grateful for
helpful conversations with Po-An Wang, Junpei Komiyama, Daniel Russo, Whitney Newey, and Pepe Montiel Olea.
This work was partially funded by ONR grant N-00014-24-12091.}
are considered a gold-standard method for causal inference and data-driven decision making in many
fields \citep{fisher1925design,imbens2015causal}. The traditional approach
to RCT designs involves pre-specifying the treatment assignment mechanism before
any data is collected; for example, in completely randomized trials (CRTs),
the actual treatment assignment is chosen uniformly at random from all possible
assignments with the same marginal treatment fractions \citep{fisher1925design}.
Recently, however, there has been growing concern that in settings where an analyst
wants to learn about multiple treatment arms and has the option to run an adaptive experiment,
standard designs such as CRTs may be inefficient relative to adaptive
designs that can use data collected early in the trial to better target
their experimentation budget. Such considerations can arise, for example,
in online marketing \citep{chapelle2011empirical}, interface design
\citep{qin2023adaptive}, job-search assistance \citep{caria2024adaptive},
or vaccine trials \citep{wu2022partial}.
It is by now clear that, in some application areas, adaptive trials can vastly
outperform non-adaptive designs \citep[e.g.,][]{chapelle2011empirical}. What's
less clear, however, is whether successful deployment of adaptive experimental
designs fundamentally relies on the use of problem-specific domain knowledge (in which
case basic CRTs would remain attractive as a robust, domain-agnostic baseline
method), or whether there exist adaptive designs that uniformly dominate CRTs \citep{qin2022open}.

Here, we provide an affirmative answer to this question in the context of
the well-known best-arm identification problem \citep{ABR-10}.\footnote{A similar problem is explored in the simulation literature, known as ranking and
selection problem \citep{glynn2004large, hong2021review}.} The goal in best-arm identification is to
deploy a treatment arm at the end of the experiment with high confidence that it
is the best, or achieves low
degradation in welfare. We show that---for this task and using a large-deviation error
metric (described further in Definition \ref{def:efficiency exponent} below)---there exist simple adaptive designs that dominate standard CRTs 
on an instance-by-instance level whenever there are at least $K \geq 3$ treatment arms to choose from. This implies that standard CRTs are not admissible among potentially adaptive randomized experiments in a sense analogous to that of
\citet{wald1949statistical}. 
Our result provides an affirmative answer to the second open problem posed in \citet{qin2022open}.

Our dominance results are achieved within a class of ``batched arm elimination'' (BAE)
designs, presented in Section \ref{sec:bae}. BAE designs sequentially discard the worst-performing arms from the experiment at pre-specified checkpoints. 
When there are only $K=2$ treatment arms, BAE designs reduce to standard CRTs. However, when $K \geq 3$ arms are available, we show that simple BAE designs can uniformly outperform CRTs. Here is an example of such BAE designs: given $T$ experimental units:
\begin{enumerate}
    \item Run a completely randomized trial with all $K$ arms on $\left\lceil \frac{K}{2(K-1)} T  \right\rceil$ units.
    \item Discard the worst-performing arm after the first batch.
    \item Run a completely randomized trial with the remaining $(K - 1)$ arms on all remaining units.
    \item Select the best of these $(K - 1)$ arms based on aggregate empirical performance across both batches.
\end{enumerate}
These dominance results are derived from an exact characterization of the large-deviation
behavior of BAE designs, which is presented in Section \ref{sec:main}.
In Section~\ref{sec:numerical}, we demonstrate our proposed design on a semi-synthetic experiment calibrated to a randomized trial by \citet{karlan2007does}.

\subsection{Problem formulation}

We frame our analysis in terms of a standard i.i.d.~sampling model for multi-armed experimentation \citep{lattimore2020bandit}.
An experimenter conducts an adaptive experiment to identify the best treatment arm among $K$ arms to deploy, following sequentially assigning these $K$ treatments to $T$ experimental units.
The potential outcome of assigning treatment $i  \in [K] \triangleq \{1,\ldots, K\}$ to experimental unit $t \in [T] \triangleq \{1,\ldots, T\}$ is a scalar random variable $Y_{t,i}$, where larger values indicate more desirable outcomes. We assume that for each treatment $i$, the potential outcomes $(Y_{t,i})_{t\in [T] }$ are drawn i.i.d. from a distribution $P(\cdot \mid \theta_i)$ with an unknown scalar parameter $\theta_i$. 
If these parameters were known, the experimenter would deploy an arm with the highest expected outcome, given by
$
\max_{i\in[K]} \intop y \cdot P(\mathrm{d}y\mid \theta_i).
$
Let $\thetabf \triangleq (\theta_1, \dots, \theta_K)$ denote the vector of unknown parameters, 
which we refer to as the \emph{problem instance}. 
Since these parameters are unknown, the experimenter interacts sequentially with  experimental units to learn which treatment arm is best.  
For the $t$-th experimental unit, the experimenter selects a treatment arm $I_{t}\in [K]$ based on the history of previously assigned treatments and observed outcomes, denoted by
$
H_{t-1}\triangleq \{I_1, Y_{1, I_1}, \ldots, I_{t-1}, Y_{{t-1},I_{t-1}}\}.
$
Importantly, only the outcome of the chosen treatment, $Y_{t, I_{t}}$, is observed; outcomes of the unselected treatments remain unknown.
The experimenter's goal is to identify and deploy the best treatment arm among $K$ arms with high confidence by the end of the experiment.

The experimenter needs to design a policy $\pi$, which is a (potentially randomized) decision rule that governs both the sequential allocation of treatment arms to $T$ experimental units and the final deployment of a treatment arm; specifically, it consists of:
\begin{enumerate}
\item An \emph{allocation rule} that sequentially assigns treatment arms based on observed outcomes.
\item A \emph{deployment rule} that selects a treatment arm for deployment after all $T$ units have been treated.
\end{enumerate}
Formally, the allocation rule is a function that maps 
the history of past allocations and outcomes, denoted by $H_{t-1}$, and the sample size $T$
to the treatment assignment $I_t$ for the $t$-th experimental unit. Additionally, after all $T$ units have been treated, the deployment rule maps the sample size $T$ and the full history $H_T$ to the final deployed arm $\hat{I}_T$.
We denote the class of all such policies by $\Pi$.

The experimenter's objective is to minimize the \emph{post-experiment utilitarian regret} of the arm deployed under policy $\pi$---also known as \emph{simple regret}, as termed by \citet{ABR-10}:
\begin{equation}
\label{eq:regret}
\mathfrak{R}^{\pi}_{\thetabf, T} \triangleq   \max_{i\in[K]}\theta_i - \E_{\thetabf,T}^{\pi}\left[\theta_{\hat{I}_T}\right].
\end{equation}
A number of authors have shown that, for analytic tractability, it is helpful to study adaptive experiments in an asymptotic regime where errors can be characterized using large-deviation methods \citep{chernoff1959sequential,glynn2004large,kaufmann2016complexity,russo2020simple}.
Here, we also leverage such asymptotics, under which
different exploration policies can be usefully compared in terms of the efficiency exponent given below. 
\begin{definition}[Efficiency exponent]
\label{def:efficiency exponent}
The efficiency exponent of a policy $\pi$ for instance $\thetabf$ is
\[
\mathfrak{e}^\pi_{\thetabf} \triangleq \liminf_{T\rightarrow \infty }  \,\, -\frac{1}{T} \ln\left(\mathfrak{R}^\pi_{\thetabf,T}\right).
\]
\end{definition}
We refer to a policy as admissible if there exists no other policy that beats it on an instance-by-instance level
(with strict inequality for some instance). 
\begin{definition}[Large-deviation admissible design]
\label{def:LD admissible}
Given a set of candidate instances $\Theta$,
a policy $\pi$ is large-deviation admissible is there is no policy $\tilde{\pi}\in \Pi$ such that
\[
\forall \thetabf\in\Theta, \quad
\mathfrak{e}_{\thetabf}^{\tilde{\pi}}  
\geq  
\mathfrak{e}_{\thetabf}^{\pi}
\quad\text{and}\quad
\exists \thetabf'\in\Theta, \quad
\mathfrak{e}_{\thetabf'}^{\tilde{\pi}}  
> 
\mathfrak{e}_{\thetabf'}^{\pi}
\]
\end{definition}

Throughout this paper, for simplicity, we will work under a setting where there is a unique best
arm and all arms have Gaussian sampling distributions with the same variance. Under this setting,
the efficiency exponent of completely randomized trials is well known. Our challenge will be to
find a policy that achieves a higher efficiency exponent on an instance-by-instance level.

\begin{assumption}
\label{assu:gauss}
Let $\sigma^2 > 0$ such that $P(\cdot \,|\, \theta_i) = \mathcal{N}(\theta_i, \, \sigma^2)$.
Given this class of distributions, we consider the set $\Theta$ of problem instances with a unique best arm,
\begin{equation}
\label{eq:instance class}
\Theta \triangleq \left\{\thetabf = (\theta_1,\ldots,\theta_K)\in \mathbb{R}^K \,\, : \,\,   I^*(\thetabf) \triangleq \argmax_{i\in[K]}\theta_i \text{ is a singleton set}\right\}.
\end{equation}
\end{assumption}

\begin{proposition}
\label{prop:exponent of uniform}
Under Assumption \ref{assu:gauss}, a completely randomized trial that (non-adaptively) uniformly allocates treatment across
$K$ available arms achieves an efficiency exponent:
\begin{equation}
\label{eq:exponent of uniform}
%\begin{split}
\mathfrak{e}_{\thetabf}^{\tt{Unif}} = 
%{\Delta_{\min}(\thetabf)^2} \,\big/\, \left(4K\sigma^2 \right), 
\frac{\Delta_{\min}(\thetabf)^2}{4K\sigma^2}
\quad\text{where}\quad
\Delta_{\min}(\thetabf) \triangleq \theta_{I^*} - \max_{i\neq I^*}\theta_i.
%\end{split}
\end{equation}
\proof
Recall that \citet[Proposition 2]{russo2020simple} establishes that
\[
\lim_{T\rightarrow \infty }  \,\, -\frac{1}{T} \ln\left(\Prob^{\tt{Unif}}_{\thetabf,T}\left(\hat{I}_T \neq I^*\right)\right) = \frac{\Delta_{\min}(\thetabf)^2}{4K\sigma^2}.
\]
Since for any policy $\pi$,
\[
\Delta_{\min}(\thetabf)\cdot \Prob^{\pi}_{\thetabf,T}\left(\hat{I}_T \neq I^*\right)
\leq \mathfrak{R}^{\pi}_{\thetabf, T} \leq 
\Delta_{\max}(\thetabf)\cdot \Prob^{\pi}_{\thetabf,T}\left(\hat{I}_T \neq I^*\right),
\]
where $\Delta_{\max} \triangleq \theta_{I^*} - \min_{i\neq I^*}\theta_i$, the equality in \eqref{eq:exponent of uniform} immediately from \citet[Proposition 2]{russo2020simple}, noting that $\Delta_{\min}(\thetabf) > 0$ by uniqueness of the best arm and that the optimality gaps $\Delta_{\min}$ and $\Delta_{\max}$ become irrelevant in the asymptotic regime
\qed
\end{proposition}

\subsection{Related work}

Pure exploration problems consist of an initial adaptive data collection phase followed by a deployment step.
Many different problem formulations fall under this broad framework. In this paper, we consider a ``fixed-budget'' model where the experiment length is given, and we seek the best possible post-experiment guarantees \citep{ABR-10, wang2023best}. Another classical model for pure exploration is the ``fixed-confidence'' model, where the
target error rate is taken as given and we seek to guarantee this error rate with the shortest possible expected experiment length \citep{chernoff1959sequential, garivier2016optimal}. Furthermore, one can use different metrics to quantify errors, including the utilitarian regret of the deployed arm \citep{kasy2021adaptive}. 
Regardless of the problem formulation, exact finite-sample analyses
for these questions present significant analytical challenges; consequently, most high-profile results in this area rely on asymptotics \citep{chernoff1959sequential, glynn2004large, kaufmann2016complexity, garivier2016optimal, russo2020simple}, as do we.

Among pure exploration problems, arguably the fixed-confidence setting has the longest history,
dating back to the classical work of \citet{chernoff1959sequential} on the sequential design of experiments, and
optimality under this model is well understood. In particular, \citet{garivier2016optimal}
demonstrate the existence of universally asymptotically optimal experimental designs under this model: There exist designs that guarantee error rate $\delta$ and whose expected stopping time as
$\delta \rightarrow 0$ has the best possible dependence for every problem instance~$\thetabf$.
\citet{qin2024optimizing} introduce a unified model that bridges the fixed-confidence setting and the classical regret minimization framework of \citet{lai1985asymptotically}, unifying results from both strands of the literature.

However, while the fixed-confidence problem seems a dual to the fixed-budget problem considered here, insights derived under the fixed-confidence model cannot be directly adapted to our setting \citep{qin2022open}. 
Unlike in the fixed-confidence model, universally asymptotically optimal policies do not exist under the fixed-budget model in full generality: \citet[Theorem 8]{degenne2023existence} show that no such policy exists for Bernoulli bandits with two arms or for Gaussian bandits with $K > e^{80/3}$ arms, and \citet{degenne2023existence} further conjectures that no universally asymptotically optimal policy exists even when $K\geq 3$ arms.
Thus, the problem of optimal experimental design under the fixed-budget model is fundamentally more complicated than that under the fixed-confidence model.

Given this context---and especially the non-existence of universally optimal designs in the fixed-budget
setting---we fall back on a follow-up question: Are there policies that at least dominate CRTs in terms
of their efficiency exponent, or, conversely, are CRTs large-deviation admissible? This question was recently
highlighted as an open problem by \citet{qin2022open};\footnote{This is the second open problem posed by \citet{qin2022open}; the first one was addressed by the results of \citet{ariu2021policychoicebestarm,degenne2023existence}.}
and, to the best of our knowledge, remained open until this paper. 

We do note that, when there are only $K = 2$ arms, CRTs are difficult to outperform---unlike the case with $K\geq 3$ arms. For two-armed Gaussian bandits,
\citet[Theorem 12]{kaufmann2016complexity} prove that Neyman allocation, i.e., CRTs with samples allocated proportionally
to arm variances, is universally asymptotically optimal (under an assumption that arm variances are known a-priori).
Meanwhile, for two-armed Bernoulli bandits, although \citet[Theorem 10]{degenne2023existence} show that no universally optimal policy exists, \citet{wang24c} prove that CRTs remain large-deviation admissible in the sense of Definition~\ref{def:LD admissible}.

The class of adaptive algorithms we design to beat CRTs under the fixed-budget model is an adaption of the
``successive rejects'' algorithm of \citep{ABR-10}, and falls under the general class of batched bandit
algorithms \citep{perchet2016batched}. One important question left open in this paper is the question of
hypothesis testing using our proposed algorithms; for example, it would useful to provide $p$-values
against the null hypothesis that the chosen arm was in fact sub-optimal. Recent advances in the literature
on inference from adaptively collected data include \citet{hadad2019confidence}, \citet{hirano2023asymptotic},
\citet{luedtke2016statistical} and \citet{zhang2020inference}.

\section{Batched arm elimination}
\label{sec:bae}

Recall that we seek to design an adaptive policy that can choose a good arm among $K$ options using $T$ datapoints.
Batched arm elimination (BAE) initializes the candidate arm set as $[K]$ and divides the whole sample of $T$ experimental units into $K-1$ batches. For each batch, BAE experiments on the treatment arms in the candidate set in a round-robin manner, and discards at the end of the batch an arm from the candidate set with the lowest empirical mean.

More specifically, BAE takes as inputs the sample size $T$ and $K-1$ instance-agnostic batch weights $\beta_{K}, \beta_{K-1},\ldots, \beta_2$, where the sample size $\beta_n\cdot T$ corresponds to the batch with the candidate set of $n$ arms.
BAE begins by experimenting on all arms in a round-robin fashion until $K$ arms have been allocated to a total of $\beta_K\cdot T$ experimental units, after which the arm with the lowest sample mean is eliminated. It then proceeds to experiment on the remaining  $K-1$ arms using another $\beta_{K-1}\cdot T$ units.
This process continues iteratively, reducing the number of arms in each batch by one, until only one arm remains.

We use the following notation throughout. Let $\Sigma_{K-1}$ be the $(K-2)$-dimensional simplex with $K-1$ entries. Given this, $(\beta_K, \beta_{K-1}, \ldots, \beta_2)\in\Sigma_{K-1}$. For $t\leq T$, the number of experimental units that receive treatment $i$ is denoted by $N_{t,i} \triangleq \sum_{\ell=1}^{t} \ind(I_\ell =i)$. When it is positive, we define the empirical mean reward as 
\begin{equation}\label{eq:posterior}
	m_{t,i} \triangleq \frac{\sum_{\ell=1}^{t} \ind(I_\ell=i) Y_{\ell,i}}{N_{t,i}}.
\end{equation}
When $N_{t,i}=0$, we let $m_{t,i} = 0$. We give pseudocode for the BAE
procedure in Algorithm~\ref{alg:bae}.

\begin{algorithm}[t]
	\centering
	\caption{Batched arm elimination}
    \label{alg:bae}
	\begin{algorithmic}[1]
		\State{\bf Input:} Sample size $T$, and batch weights $(\beta_K, \beta_{K-1}, \ldots, \beta_2)\in\Sigma_{K-1}$
		\State {\bf Initialize:} Number of remaining arms $n \gets K$, and candidate set $C \gets [K]$
		\For{$t=1,\ldots, T$}
            \State Assign $I_t \in \argmin_{i\in C} N_{t-1,i}$ to the $t$-th individual
            \State Update $\{N_{t, i}\}_{i\in C}$ and $\{m_{t, i}\}_{i\in C}$
            \If{$r \geq  2$ and $t = (\beta_K + \cdots + \beta_r)T$}
            %\State Drop an arm with the lowest MLE from $C$; $j\gets j-1$
            \State Remove $\ell_n\in\argmin_{i\in C}m_{t,i}$ from $C$, i.e.,
            $C\gets C\setminus \{\ell_n\}$, and $n\gets n-1$
            %\check{I}_t
            \EndIf
		\EndFor 
  \State{\bf Output:} The only remaining arm in $C$
	\end{algorithmic}
\end{algorithm}

We note that BAE is a direct generalization of the ``successive rejects'' algorithm proposed by \citet{ABR-10}.
Successive rejects is BAE with the following batch weights:
\[
\left(\beta_K, \beta_{K-1}, \ldots, \beta_3\right) = \left(1, \frac{1}{K},\ldots,\frac{1}{4}\right) \cdot \frac{1}{\overline{\ln}(K)}
%\beta_K = \frac{1}{\overline{\ln}(K)},\quad \beta_{K-1} = \frac{1}{K\overline{\ln}(K)}, \ldots, \quad\beta_3 = \frac{1}{4\overline{\ln}(K)}, 
\quad\text{and}\quad\beta_2 = 1 - \sum_{n = 3}^K \beta_{n}.
\]
where $\overline{\ln}(K) = \frac{1}{2} + \sum_{i = 2}^K \frac{1}{i}$.
\citet{ABR-10} did not investigate uniform-dominance
results as we do here. Furthermore, the set of BAE procedures we show dominate
uniform sampling in fact does not include the original successive rejects algorithm. It can be verified that uniform allocation outperforms successive rejects in instances where all suboptimal arms are identical.
Additionally, we note that CRTs are a speical case of BAE designs with batch weights $(1,0,\ldots,0)\in\Sigma_{K-1}$. In other words, CRTs consists of a single batch that includes the entire sample size.

Our main result shows that there exist instance-agnostic batch weights such that BAE achieves a strictly higher efficiency exponent than uniform allocation for any problem instance. %$\thetabf\in\Theta$.

\section{CRT-dominating batched arm elimination}
\label{sec:main}

In this section, we demonstrate the universal superiority BAE designs over CRTs in terms of the efficiency exponent, providing certain sufficient conditions are met. We begin by deriving an exact characterization of the large-deviation behavior of all BAE designs.
Our characterization result extends the results of \citet{wang2023best, wang24c}, which focused on bounded observations and aimed to minimize the probability of misidentifying the best arm, with an emphasis on the successive rejects algorithm \citep{ABR-10} and its two variants proposed in \citet{wang2023best}. In contrast, we consider unbounded Gaussian observations and aim to minimize the utilitarian regret, studying all BAE designs with any batch weights $(\beta_K, \beta_{K-1}, \ldots, \beta_2)\in\Sigma_{K-1}$.

Recall that in BAE designs (Algorithm \ref{alg:bae}), we use $n\in\{K,\ldots,2\}$ to denote the number of remaining arms. 
Let $\mathcal{J}_{\thetabf,n}$ denote the collection of all sets with cardinality $n$ that includes the best arm $I^* = I^*(\thetabf)$:
\begin{equation}
\label{eq:set of sets}
\mathcal{J}_{\thetabf,n} \triangleq \{J\subseteq [K]:\,\,\left|J\right|=n,\,\, I^* \in J\},
\end{equation}
which depends on $\thetabf$ since it is defined based on its best arm $I^*$. 
Consider such a set $J\in \mathcal{J}_{\thetabf,n}$; we define the set of instances where $I^*$ is the worst-performing arm in this set:
\begin{equation}
\Lambda_{\thetabf, J} \triangleq \left\{\lambdabf\in\mathbb{R}^K:\,\,
%\max_{i\notin J}\lambda_i < 
\lambda_{I^*} \leq \min_{i\in J}\lambda_i
\right\}.
\end{equation}
Building on the previous definitions, we introduce a quantity that captures the minimal information required for the best arm $I^*$ to be eliminated at the end of a batch, starting with $n$ arms, by the remaining $n-1$ arms:
\begin{equation}
\label{eq:gamma}
\Gamma_{\thetabf, n} \triangleq \min_{J\in \mathcal{J}_{\thetabf, n}}
%\inf_{\substack{\lambdabf\in \mathbb{R}^K: \\ \lambda_{I^*}\le \min_{i \in J} \lambda_i}} 
\inf_{\lambdabf\in \Lambda_{\thetabf, J}}
\sum_{i \in J}d(\lambda_i,\theta_i),
\end{equation}
where $d(\lambda, \theta)$ is the Kullback–Leibler (KL) divergence between two distributions parameterized by $\lambda$ and $\theta$. For two Gaussian distributions with means $\lambda$ and $\theta$ and common variance $\sigma^2$, the KL divergence $d(\lambda, \theta) = \frac{1}{2\sigma^2}(\lambda - \theta)^2$.

In addition to the minimal information quantity $\Gamma_{\thetabf, n}$, we compute the proportion (of the sample size $T$) allocated to the arm eliminated at the end of the batch, starting with $n$ arms:
\begin{equation}
\label{eq:proportion}
w_n \triangleq \frac{\beta_K}{K} + \frac{\beta_{K-1}}{K-1} + \cdots + \frac{\beta_n}{n},
\end{equation}
where the first term arises from the fact that, in the first batch, the proportion $\beta_K$ (of the sample size $T$) is uniformly allocated across $K$ arms, and the remaining terms follow the same logic.
These quantities introduced above characterize the efficiency exponent of the BAE designs, as established in Lemma~\ref{lem:bae}  below. The proof of Lemma \ref{lem:bae} is provided in Appendix~\ref{sec:proof of main}.

\begin{lemma}
\label{lem:bae}
Under Assumption \ref{assu:gauss} and
for batched arm elimination with batch weights $(\beta_K, \beta_{K-1}, \ldots, \beta_2)\in\Sigma_{K-1}$ (Algorithm \ref{alg:bae}), 
\[
\mathfrak{e}_{\thetabf}^{\tt{BAE}}  \geq \min_{n \in  \{K,K-1,\ldots,2\}}\,\, w_n\Gamma_{\thetabf,n},
\]
where $w_n$ and $\Gamma_{\thetabf,n}$ are defined in \eqref{eq:proportion} and \eqref{eq:gamma}, respectively.
\end{lemma}

Our next task is to establish sufficient conditions for BAE designs being universally dominating the CRT.
We derive the following lower bound on the information quantity $\Gamma_{\thetabf,n}$ in \eqref{eq:gamma}, by analyzing different instance configurations respectively.

\begin{lemma}[A lower bound on minimal information $\Gamma_{\thetabf,n}$]
\label{lem:lower bound}
For any $n\in \{K, \ldots, 2\}$, 
\[
\Gamma_{\thetabf,n}\geq \frac{n-1}{n}\frac{\Delta_{\min}(\thetabf)^2}{2\sigma^2}
\quad\text{where}\quad
\Delta_{\min}(\thetabf) = \theta_{I^*} - \max_{i\neq I^*}\theta_i.
\]
The inequality becomes equality for the instances such that the suboptimal arms are the same, i.e., $\theta_i = \theta_j$ for any $i,j\neq I^*$.
\end{lemma}

The proof of Lemma \ref{lem:lower bound} is provided in Section \ref{sec:lower bound}. 
By integrating this lower bound on the minimal information $\Gamma_{\thetabf,n}$ into the efficiency exponent in Lemma~\ref{lem:bae}, we obtain the lower bound for efficiency exponent of BAE designs.

\begin{corollary}[A lower bound on BAE's efficiency exponent]
Under Assumption \ref{assu:gauss} and
for batched arm elimination with batch weights $(\beta_K, \beta_{K-1}, \ldots, \beta_2)\in\Sigma_{K-1}$ (Algorithm \ref{alg:bae}), 
\[
\mathfrak{e}_{\thetabf}^{\tt{BAE}} \geq \frac{\Delta_{\min}(\thetabf)^2}{2\sigma^2} \min_{n\in\{K,\dots,2\}} w_n\frac{n-1}{n}.
\]
\end{corollary}

By comparing this lower bound for efficiency exponent of BAE designs with the efficiency exponent of CRTs in Proposition \ref{prop:exponent of uniform}, we immediately derive the sufficient conditions for BAE designs to universally outperforms CRTs, based on the batch weights of BAE designs.
\begin{Theorem}
[Sufficient conditions]
\label{theo:suff}
Under Assumption \ref{assu:gauss} and
for batched arm elimination with batch weights $(\beta_K, \beta_{K-1}, \ldots, \beta_2)\in\Sigma_{K-1}$ (Algorithm \ref{alg:bae}),
\begin{equation}
\label{eq:condition}
\min_{n\in\{K,\ldots,2\}}\,\, w_n \frac{n-1}{n} > \frac{1}{2K}
\quad\implies\quad
\mathfrak{e}_{\thetabf}^{\tt{BAE}}  
> 
\mathfrak{e}_{\thetabf}^{\tt{Unif}}, \quad\forall \bm{\theta}\in\Theta.
\end{equation}
\end{Theorem}

Recall that $w_n$, defined in \eqref{eq:proportion}, is the proportion (of the sample size $T$) allocated to the arm eliminated at the end of the batch starting with $n$ arms. 
As the number of remaining arms $n$ decreases, the proportion $w_n$ increases, while the multiplier $\frac{n-1}{n}$ decreases. 
The sufficient condition above requires that the proportions allocated to all arms, weighted by their respective multipliers, exceed half the proportion each arm would receive under uniform allocation or CRTs.

We present simple BAE designs consisting of only two batches, with batch weights satisfying the sufficient conditions in \eqref{eq:condition}. 

\begin{example}[Two-batch CRT-dominating BAE]
\label{coro:example}
%In the setting of Theorem \ref{theo:suff},
Consider a two-batch policy that eliminate $s\in [K-2]$ arms after the first batch, and the other $(K-1-s)$ arms at the end of the second batch (i.e., at time $T$):
\[
\beta_{K-1} = \beta_{K-2} = \cdots = \beta_{K-s + 1} = 0,
\quad
\beta_{K-s} = 1 - \beta_K,
\quad
\beta_{K-s-1} = \cdots = \beta_2 = 0.
\]
Under this, the sufficient conditions for CRT-dominance become
\[
\beta_K > \frac{1}{2} + \frac{1}{2(K-s)}.
\]
The example presented in the introduction corresponds to the case $s=1$, where only one arm is dropped after the first batch.

As the number of arms eliminated after the first batch increases, the minimal required first-batch size also increases. This is intuitive, as BAE must allocate a larger initial batch size to safely eliminate more arms but still achieve universal improvement. 
On the other hand, if the number of arms eliminated after the first batch is fixed while the total number of arms grows large, the minimal required size of the first batch approaches half of the total sample size.

%\end{remark}

\end{example}

\section{Numerical Example}
\label{sec:numerical}

\citet{karlan2007does} report results on an experiment to test whether
matching gifts increase charitable giving. Their experiment includes 4
arms: A control arm, and 3 treatment arms that offer 1:1, 2:1 and 3:1
matches to potential donors (a k:1 match promises that each \$1 given will
be matched by a \$k gift from a different donor). The outcome we are
interested in was the total amount given. This is a sparse and
skewed outcome: Only 2\% of prospective donors give anything (i.e., 98\%
of the outcomes are 0), and conditionally on donating the mean donation
is \$44 with a standard deviation of \$42. The distribution of the data
presents a marked departure of the Gaussianity assumption use in our
formal results, and presents us with an opportunity to investigate
practical robustness of our algorithm to distributions one might face
in real-world applications.

The original experiment of \citet{karlan2007does} had a sample size of
$n = 50,083$. Here, we start by non-parametrically fitting the data-generating
distribution. For each arm we separately tally the fraction of zero outcomes
and fit the density of log donation amounts for non-zero donations via kernel
density estimation; and combine these to obtain a zero-inflated and
skewed density for the outcomes themselves. We then simulate data from this
fitted distribution to compare the performance of the following two algorithms
for different sample sizes $T$:
\begin{itemize}
    \item Completely randomized trial: We simply allocate $T/4$ samples to each arm.
    \item The variant of batched arm elimination described in Corollary \ref{coro:example} with $s = 1$: We run a completely randomized trial on all arms using $\frac23 T$ of the data, eliminate the worst-performing arm, and run a completely randomized trial on the remaining 3 arms using the rest of the data.
\end{itemize}
Throughout, we pick $T$ so that no rounding is required.

\begin{figure}[t]
    \centering
    \begin{tabular}{cc}
    \includegraphics[width=0.45\linewidth]{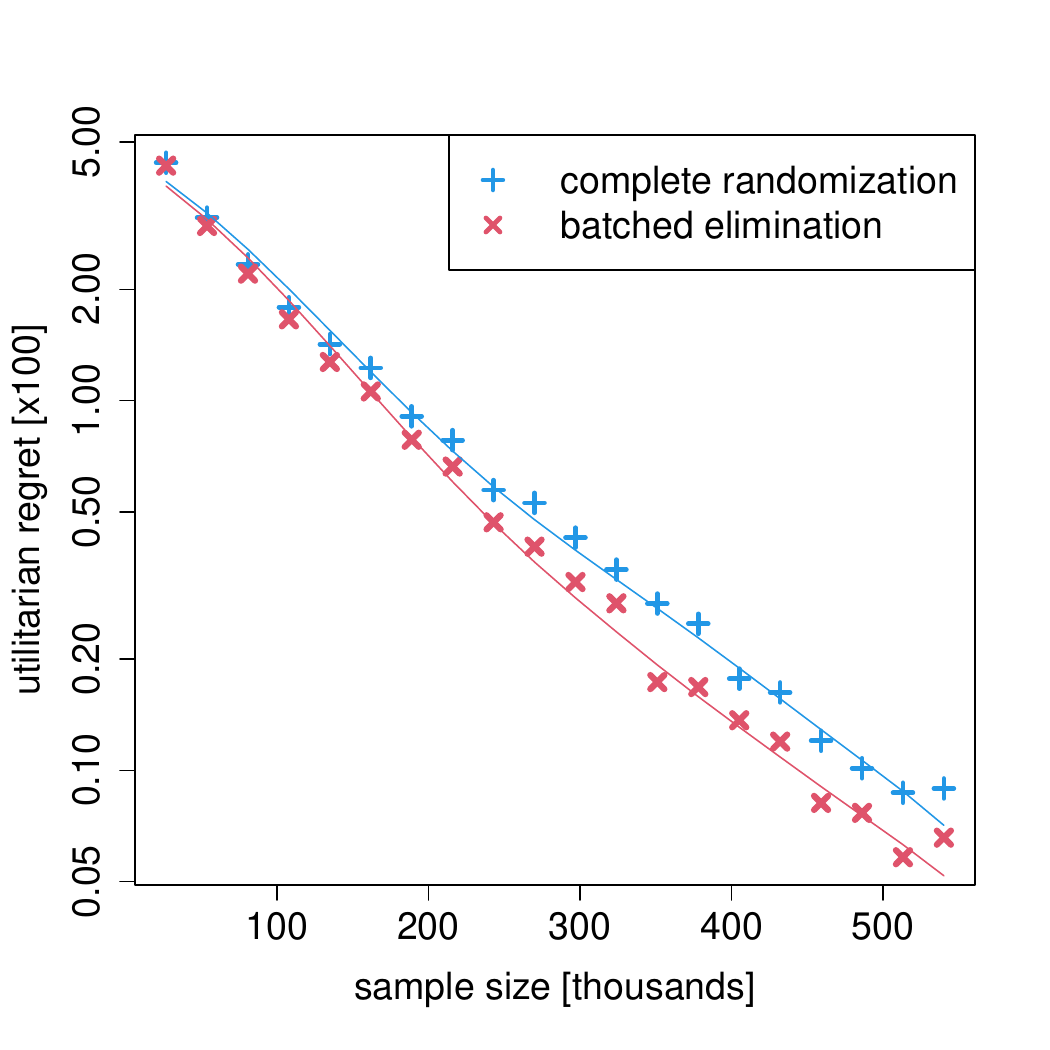} &
    \includegraphics[width=0.45\linewidth]{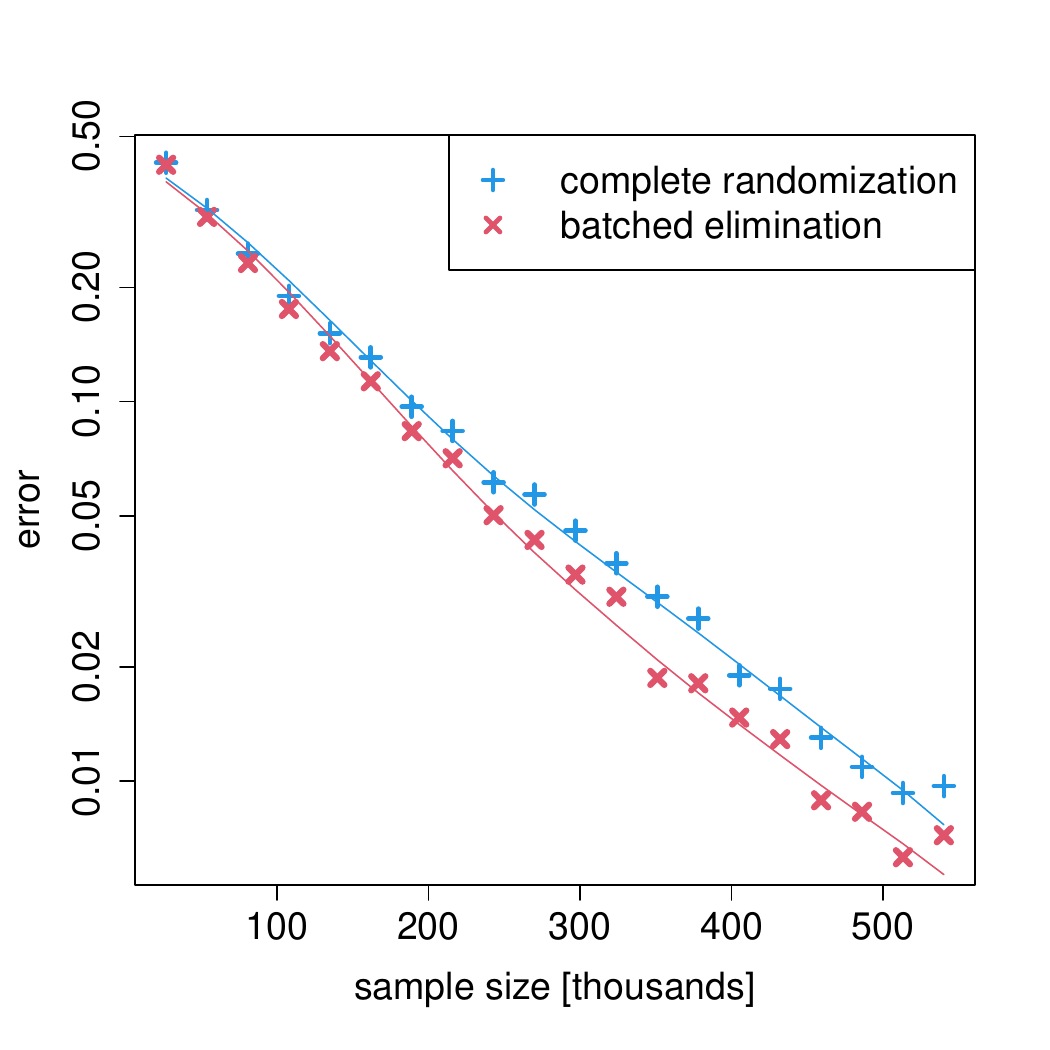}
    \end{tabular}
    \caption{Utilitarian regret (left) and arm-selection error rates (right) for both completely randomized trials and batched arm elimination, as a function of the total sample size $T$. All results are aggregated across 10,000 simulation replications. Smoothers are added for visualization purposes only.}
    \label{fig:err_example}
\end{figure}

Results are shown in Figure \ref{fig:err_example}. We see that both experimental designs perform comparably with smaller sample sizes (in which case they both frequently make errors); however, as the sample size grows, the error of batched arm elimination decays faster than than of the non-adaptive design. Thus, at least in this design, batched arm elimination appears to give a ``safe'' improvement over non-adaptive experiments: We can benefit from adaptive in high-signal regimes without compromising performance in low-signal ones.

\newpage
\appendix

\section{Proof of Lemma \ref{lem:bae}}
\label{sec:proof of main}

For any policy $\pi$, we have the following bounds on its utilitarian regret $\mathfrak{R}^{\pi}_{\thetabf, T}$, defined in \eqref{eq:regret}:
\[
\Delta_{\min}(\thetabf)\cdot \Prob^{\pi}_{\thetabf,T}\left(\hat{I}_T \neq I^*\right)
\leq \mathfrak{R}^{\pi}_{\thetabf, T} \leq 
\Delta_{\max}(\thetabf)\cdot \Prob^{\pi}_{\thetabf,T}\left(\hat{I}_T \neq I^*\right),
\]
where $\Delta_{\min} = \theta_{I^*} - \max_{i\neq I^*}\theta_i$ and $\Delta_{\max} = \theta_{I^*} - \min_{i\neq I^*}\theta_i$.
We note that $\Delta_{\min}(\thetabf) > 0$ by uniqueness of the best arm (Assumption \ref{assu:gauss}), and that the optimality gaps become irrelevant in the asymptotic regime. Hence, to prove Lemma \ref{lem:bae}, it suffices to establish the following result.

\begin{lemma}
\label{lem:bae prob of error}
Under Assumption \ref{assu:gauss} and
for batched arm elimination with batch weights $(\beta_K, \beta_{K-1}, \ldots, \beta_2)\in\Sigma_{K-1}$ (Algorithm \ref{alg:bae}), 
\[
\lim_{T\rightarrow \infty }  \,\, -\frac{1}{T} \ln\left(\Prob^{\tt{BAE}}_{\thetabf,T}\left(\hat{I}_T \neq I^*\right)\right)  \geq \min_{n \in  \{K,K-1,\ldots,2\}}\,\, w_n\Gamma_{\thetabf,n},
%\textnormal{SR+}
\]
where $w_n$ and $\Gamma_{\thetabf,n}$ are defined in \eqref{eq:proportion} and \eqref{eq:gamma}, respectively.
\end{lemma}

By the construction of BAE designs (Algorithm \ref{alg:bae}), we have
\begin{align*}
\Prob^{\tt{BAE}}_{\thetabf,T}\left(\hat{I}_T\neq I^*\right) 
=  \sum_{n\in\{K,\ldots,2\}}\Prob^{\tt{BAE}}_{\thetabf,T}\left(\ell_n = I^*\right),
\end{align*}
where $\ell_n$ is the arm eliminated at the end of the batch starting with $n$ arms.
Hence,
\[
\Prob^{\tt{BAE}}_{\thetabf,T}\left(\hat{I}_T\neq I^*\right) 
\leq (K-1)\max_{n\in\{K,\ldots,2\}}\Prob^{\tt{BAE}}_{\thetabf,T}\left(\ell_n = I^*\right),
\]
and thus
\begin{align}
\liminf_{T\rightarrow \infty }  \,\, -\frac{1}{T} \ln\left(\Prob^{\tt{BAE}}_{\thetabf,T}\left(\hat{I}_T \neq I^*\right)\right) 
&\geq \lim_{T\rightarrow \infty }  \,\, -\frac{1}{T}\ln\left(\max_{n\in\{K,\ldots,2\}}\Prob^{\tt{BAE}}_{\thetabf,T}\left(\ell_n = I^*\right)\right) \nonumber\\
&= \liminf_{T\rightarrow \infty }  \,\, \min_{n\in\{K,\ldots,2\}}-\frac{1}{T}\ln\left(\Prob^{\tt{BAE}}_{\thetabf,T}\left(\ell_n = I^*\right)\right) \nonumber\\
&= \min_{n\in\{K,\ldots,2\}}\liminf_{T\rightarrow \infty }  \,\, -\frac{1}{T}\ln\left(\Prob^{\tt{BAE}}_{\thetabf,T}\left(\ell_n = I^*\right)\right),
\label{eq:prob of error lower bound}
\end{align}
where the last equality holds since the set of potential values of $n$ is finite.

By \eqref{eq:prob of error lower bound}, establishing the lower bound in Lemma \ref{lem:bae prob of error} reduces to prove the following result.

\begin{lemma}
\label{lem:eliminate best}
For $n\in\{K,\ldots,2\}$,
\[
\liminf_{T\rightarrow \infty }  \,\, -\frac{1}{T} \ln\left(\Prob^{\tt{BAE}}_{\thetabf,T}\left(\ell_n = I^*\right)\right) \geq w_n\Gamma_{\thetabf,n}.
\]
\end{lemma}

\begin{proof}
We have
\begin{equation}
\label{eq:eliminate best}
\Prob^{\tt{BAE}}_{\thetabf,T}\left(\ell_n = I^*\right)
= \sum_{J \in \mathcal{J}_{\thetabf,n}} \Prob^{\tt{BAE}}_{\thetabf,T}\left(\ell_n = I^*, C_n = J\right)
\leq |\mathcal{J}_{\thetabf,n}| \cdot \max_{J \in \mathcal{J}_{\thetabf,n}} \Prob^{\tt{BAE}}_{\thetabf,T}\left(\ell_n = I^*, C_n = J\right).
\end{equation}
Fix $J \in \mathcal{J}_{\thetabf,n}$. Since the arm $\ell_n$ is eliminated at the end of $T_n \triangleq (\beta_K+\cdots + \beta_n)\cdot T$ timesteps, the event $\{\ell_n = I^*, C_n = J\}$ implies that
$\{\bm{m}_{T_n}\in M, \bm{p}_{T_n}\in P\}$, where
\[
M = \left\{\lambdabf\in\mathbb{R}^K: \lambda_{I^*}\leq \lambda_i,\forall i \in J\right\}
\quad\text{and}\quad
P = \left\{\bm{p}\in\Sigma_K: p_i = \tilde{p},\forall i\in J\right\}.
\]
where $\tilde{p} = \frac{\left(\frac{\beta_K}{K} + \cdots + \frac{\beta_n}{n}\right)T}{T_n} = \frac{\frac{\beta_K}{K} + \cdots + \frac{\beta_n}{n}}{\beta_K + \cdots + \beta_n}$.
That is, 
\[
\Prob^{\tt{BAE}}_{\thetabf,T}\left(\ell_n = I^*, C_n = J\right) \leq 
\Prob^{\tt{BAE}}_{\thetabf,T}\left(\bm{m}_{T_n} \in M, \bm{p}_{T_n}\in P\right),
\]
and thus
\begin{align*}
\liminf_{T\rightarrow \infty }  \,\, -\frac{1}{T} 
\ln\left( \Prob^{\tt{BAE}}_{\thetabf,T}\left(\ell_n = I^*, C_n = J\right) \right)
&\geq 
\liminf_{T\rightarrow \infty }  \,\, -\frac{1}{T} 
\ln\left( \Prob^{\tt{BAE}}_{\thetabf,T}\left(\bm{m}_{T_n}\in M, \bm{p}_{T_n} \in P\right) \right) \\
&= \left(\beta_K + \cdots + \beta_n \right)\liminf_{T\rightarrow \infty }  \,\, -\frac{1}{T_n} 
\ln\left( \Prob^{\tt{BAE}}_{\thetabf,T}\left(\bm{m}_{T_n}\in M, \bm{p}_{T_n} \in P\right) \right) \\
&\geq  \left(\beta_K + \cdots + \beta_n \right)\inf_{\bm{p}\in \cl(P)}F_{\thetabf,M}(\bm{p})\\
&\geq \left(\beta_K + \cdots + \beta_n \right) \tilde{p}
\inf_{\lambdabf \in \cl (M) } \sum_{i \in J}d(\lambda_i,\theta_i)\\
&= \left(\frac{\beta_K}{K} + \cdots + \frac{\beta_n}{n}\right)
\inf_{\lambdabf \in \cl (M) } \sum_{i \in J}d(\lambda_i,\theta_i) \\
&=w_n\Gamma_{\thetabf,n},
\end{align*}
where the second inequality uses Theorem \ref{thm:LD}.
Applying the inequality in \eqref{eq:eliminate best} completes the proof.
\end{proof}

\subsection{Large-deviation results}
The proof of Lemma \ref{lem:eliminate best} above relies on the following large-deviation results, which extend \citet[Theorem 1]{wang2023best} for Gaussian distributions.
\begin{Theorem}%[Theorem 1 in \citet{wang2023best}]
\label{thm:LD}
Fix $\thetabf\in\Theta$. For a non-anticipating algorithm, if the empirical allocation sequence $\{\bm{p}_t\}_{t\geq 1}$ satisfies the large deviation principle (LDP) upper bound with rate function $I$, then 
\begin{enumerate}
    \item $\{\bm{m}_t\}_{t\geq 1}$ satisfies the LDP lower bound with rate function 
    \[
    \lambdabf \mapsto \min_{\bm{p}\in \Sigma_K}\max\{ \Psi_{\thetabf}(\lambdabf,\bm{p}),I(\bm{p})\},
    \]
    where $\Psi_{\thetabf}(\lambdabf,\bm{p}) \triangleq \sum_{i=1}^K p_i \cdot d(\lambda_i ,\theta_i)$;
    \item 
    Given $n\in\{2,\ldots, K\}$, consider a set $J\subseteq \mathbb{R}^K$ such that $|J| = n$ and $I^*\in J$. Then for the set 
    \[
    M \triangleq \left\{\lambdabf\in\mathbb{R}^K:\,\,  \lambda_{I^*}\leq \lambda_i, \,\, \forall i \in J\right\}
    \]
    and any Borel subset $P\subseteq \Sigma_K$,
    \[
    \liminf_{t\rightarrow \infty }  \,\,  - \frac{1}{t} \ln 
    \left(\Prob_{\thetabf}\left(  \bm{m}_t\in  M, \,\, \bm{p}_t\in P\right)\right) \geq  \inf_{\bm{p}\in \cl(P)} \max\left\{ F_{\thetabf,M}(\bm{p}),I(\bm{p})\right\},
    \]
    where $F_{\thetabf,M}(\bm{p})\triangleq\inf_{\lambdabf \in \cl (M) }\Psi_{\thetabf}(\lambdabf,\bm{p})$,
    and $\cl (A)$ denotes the closure of the set $A$.
\end{enumerate}
\end{Theorem}

\section{Proof of Lemma \ref{lem:lower bound}}
\label{sec:lower bound}

Fix $n \in \{K,\ldots, 2\}$. For $\mu_1 > \mu_2\geq \cdots \geq \mu_n$, we define
 \begin{align}\label{opt:1}
	\Psi(\mu_1,\mu_2,\ldots,\mu_n) \triangleq 
\inf_{\substack{(\lambda_1,\ldots,\lambda_n)\in \mathbb{R}^n: \\ \lambda_{1}\le \min_{i \in [n]} \lambda_i}} 
    \sum_{i=1}^n (\lambda_i-\mu_i)^2.
\end{align}
Given this, we can write
\[
\Gamma_{\thetabf,n} = \min_{J\in \mathcal{J}_{\thetabf, n}} \inf_{\lambdabf\in \Lambda_{\thetabf, J}}
\sum_{i \in J}d(\lambda_i,\theta_i) = \frac{1}{2\sigma^2}\min_{J\in \mathcal{J}_{\thetabf, n}} \Psi\left(\theta_{I^*}, \theta_{(2)_J},\ldots, \theta_{(n)_J}\right),
\]
where $(i)_J$ denotes the index of the $i$-th largest element in the set $J$ for $i=2,\ldots,n$; note that the largest element in $J$ is $I^*$.

To prove Lemma \ref{lem:lower bound}, it suffices to prove the following result: 
\begin{lemma}[Bounds on $\Psi(\mu_1,\mu_2,\ldots,\mu_n)$]
\label{lem:xi lower bound}
The value $\Psi(\mu_1,\mu_2,\ldots,\mu_n)$ can be lower bounded as follows,
\[
\Psi(\mu_1,\mu_2,\ldots,\mu_n) \geq \frac{n-1}{n}(\mu_1 - \mu_2)^2,
\]
where the inequality becomes equality when $\mu_2 = \cdots = \mu_n$.
On the other hand, $\Psi(\mu_1,\mu_2,\ldots,\mu_n)$ can be upper bounded:
\[
\Psi(\mu_1,\mu_2,\ldots,\mu_n)  < \sum_{i=2}^n (\mu_1 - \mu_i)^2.
\]
\end{lemma}

%%%
The proof starts with the calculation of $\Psi(\mu_1,\mu_2,\ldots,\mu_n)$, which uses the following result in \citet{wang2023best}:
\begin{proposition}[Proposition 1 in \citet{wang2023best}]
\label{prop:xi}
	The value $\Psi(\mu_1,\mu_2,\ldots,\mu_n)$ can be calculated as follows,
	$$
	\Psi(\mu_1,\mu_2,\ldots,\mu_n)=\left\{
	\begin{array}{ll}
	\sum_{i=1,n} \left(\mu_i-\frac{\mu_1+\mu_n}{2}\right)^2,&\hbox{ if }\mu_{n-1}\ge \frac{\mu_1+\mu_n}{2},\\
	\sum_{i=1,n-1,n} \left(\mu_i-  \frac{\mu_1+\mu_{n-1}+\mu_n}{3} \right)^2,&\hbox{ if }\mu_{n-1}< \frac{\mu_1+\mu_n}{2},\,\,\mu_{n-2}\ge \frac{\mu_1+\mu_{n-1}+\mu_n}{3}, \\
	\vdots&\vdots\\
	\sum_{i=1}^n \left(\mu_i -  \frac{\sum_{j=1}^n\mu_j}{n} \right)^2,&\hbox{ if }\mu_{n-1}< \frac{\mu_1+\mu_n}{2},\,\,\ldots,\,\, \mu_2< \frac{\mu_1+\mu_3+\cdots+\mu_n}{n-1}.
	\end{array}
	\right.
	$$
\end{proposition}

We can simplify the above formulas by introducing $\Delta_i = \mu_1 - \mu_i$ for $i=1, 2,\ldots, n$.

\begin{corollary}\label{cor:xi}
	The value $\Psi(\mu_1,\mu_2,\ldots,\mu_n)$ can be rewritten as follows,
	$$
	\Psi(\mu_1,\mu_2,\ldots,\mu_n)=\left\{
	\begin{array}{ll}
	\frac{\Delta_n^2}{2},&\hbox{ if }\Delta_{n-1}\le \frac{\Delta_n}{2},\\
	\left(\sum_{i=n-1}^n\Delta_{i}^2\right) - \frac{\left(\sum_{i=n-1}^n\Delta_{i}\right)^2}{3},&\hbox{ if }\Delta_{n-1}> \frac{\Delta_n}{2},\,\,\Delta_{n-2}\le \frac{\sum_{i=n-1}^n\Delta_{i}}{3}, \\
	\vdots&\vdots\\
	 \left(\sum_{i=2}^n\Delta_{i}^2\right) - \frac{\left(\sum_{i=2}^n\Delta_{i}\right)^2}{n},&\hbox{ if }\Delta_{n-1} > \frac{\Delta_n}{2},\,\,\ldots,\,\,
  \Delta_2 > \frac{\sum_{i=3}^n\Delta_i}{n-1}.
	\end{array}
	\right.
	$$
\end{corollary}

\begin{proof}[Proof of Corollary \ref{cor:xi}]
If the $m$-th condition holds where $m \in [n-1]$, we have
\begin{align*}
\Psi(\mu_1,\mu_2,\ldots,\mu_n) &= \sum_{i=1,n-m+1,\ldots,n} \left(\mu_i-  \frac{\mu_1 + \sum_{j = n-m+1}^n\mu_j}{m+1} \right)^2 \\
&= \sum_{i=1,n-m+1,\ldots,n} \left(\Delta_i - \frac{\sum_{j = n-m+1}^n\Delta_j}{m+1}\right)^2.
\end{align*}
For notational convenience, write $a = \frac{\sum_{j = n-m+1}^n\Delta_i}{m+1}$. Then
\begin{align*}
\Psi(\mu_1,\mu_2,\ldots,\mu_n) &= a^2 + \sum_{i = n-m+1}^n \left(\Delta_i - a\right)^2 \\
&= a^2 + m a^2  +  \sum_{i = n-m+1}^n\Delta_i^2 - 2a\sum_{i = n-m+1}^n\Delta_i\\
&= \sum_{i = n-m+1}^n\Delta_i^2 - (m+1)a^2 \\
&= \sum_{i = n-m+1}^n\Delta_i^2 - \frac{\left(\sum_{i = n-m+1}^n\Delta_i\right)^2}{m+1}.
\end{align*}
\end{proof}

Now we are ready to complete the proof of Lemma \ref{lem:xi lower bound}.
\begin{proof}[Proof of Lemma \ref{lem:xi lower bound}]
The upper bound follows directly from the formulas in Corollary \ref{cor:xi}.

Now we are going to prove the lower bound. If the $m$-th condition in Corollary~\ref{cor:xi} holds where $m\in[n-2]$, we have
\begin{align*}
\left(\sum_{i=n-m+1}^n\Delta_{i}^2\right) - \frac{\left(\sum_{i=n-m+1}^n\Delta_{i}\right)^2}{m+1}
&\geq \frac{\left(\sum_{i=n-m+1}^n\Delta_{i}\right)^2}{m} - \frac{\left(\sum_{i=n-m+1}^n\Delta_{i}\right)^2}{m+1} \\
&= \frac{\left(\sum_{i=n-m+1}^n\Delta_{i}\right)^2}{m(m+1)}\\
&\geq \frac{(m+1)\Delta_{n-m}^2}{m}
> \Delta_2^2,
\end{align*}
where the second-to-last inequality applied the $m$-th condition: $\Delta_{n-m} \leq \frac{\sum_{i=n-m+1}^n\Delta_{i}}{m+1}$.

On the other hand, if the last condition in Corollary~\ref{cor:xi} holds, applying the Cauchy–Schwarz inequality gives
\begin{align*}
\left(\sum_{i=2}^n\Delta_{i}^2\right) - \frac{\left(\sum_{i=2}^n\Delta_{i}\right)^2}{n}
&\geq \frac{\left(\sum_{i=2}^n\Delta_{i}\right)^2}{n-1} - \frac{\left(\sum_{i=2}^n\Delta_{i}\right)^2}{n} \\
&= \frac{\left(\sum_{i=2}^n\Delta_{i}\right)^2}{(n-1)n}\\
&\geq \frac{n-1}{n}\Delta_2^2,
\end{align*}
where the last inequality simply uses $\Delta_{i}\geq \Delta_2$ for any $i\in\{2,\ldots,n\}$.

Combining the two cases completes the inequality in Lemma \ref{lem:xi lower bound}. The equality holds if and only if $\Delta_2 = \cdots = \Delta_n$.
\end{proof}

\singlespacing
	\bibliographystyle{plainnat}
	\bibliography{references}

\end{document}